\documentclass[twocolumn, switch]{article} 

\usepackage{preprint}

\usepackage{amsmath, amsthm, amssymb, amsfonts}
\usepackage{graphicx}
\usepackage{natbib}

\usepackage[utf8]{inputenc}	
\usepackage[T1]{fontenc}	
\usepackage{xcolor}		
\usepackage[colorlinks = true,
            linkcolor = purple,
            urlcolor  = blue,
            citecolor = cyan,
            anchorcolor = black]{hyperref}	
\usepackage{booktabs} 		
\usepackage{nicefrac}		
\usepackage{microtype}		
\usepackage{lineno}		
\usepackage{float}			

\newcommand{\calI}{\mathcal I}
\newcommand{\calA}{\mathcal A}
\newcommand{\calG}{\mathcal G}

\newcommand{\CR}{\mathrm{CR}}
\newcommand{\OPT}{\mathrm{OPT}}
\newcommand{\Gap}{\mathrm{Gap}}
\newcommand{\Adv}{\mathrm{ADV}}
\newcommand{\RLA}{\mathrm{RLA}}
\newcommand{\FLA}{\mathrm{FLA}}
\usepackage{algorithm}
\usepackage{algpseudocode}

\usepackage{lipsum}		

\usepackage{newfloat}
\DeclareFloatingEnvironment[name={Supplementary Figure}]{suppfigure}
\usepackage{sidecap}
\sidecaptionvpos{figure}{c}

\usepackage{titlesec}
\titlespacing\section{0pt}{12pt plus 3pt minus 3pt}{1pt plus 1pt minus 1pt}
\titlespacing\subsection{0pt}{10pt plus 3pt minus 3pt}{1pt plus 1pt minus 1pt}
\titlespacing\subsubsection{0pt}{8pt plus 3pt minus 3pt}{1pt plus 1pt minus 1pt}

\usepackage{tikz,xcolor,hyperref}

\definecolor{lime}{HTML}{A6CE39}
\DeclareRobustCommand{\orcidicon}{
	\begin{tikzpicture}
	\draw[lime, fill=lime] (0,0)
	circle [radius=0.16]
	node[white] {{\fontfamily{qag}\selectfont \tiny ID}};
	\draw[white, fill=white] (-0.0625,0.095)
	circle [radius=0.007];
	\end{tikzpicture}
	\hspace{-2mm}
}
\foreach \x in {A, ..., Z}{\expandafter\xdef\csname orcid\x\endcsname{\noexpand\href{https://orcid.org/\csname orcidauthor\x\endcsname}
			{\noexpand\orcidicon}}
}

\usepackage{graphicx}
\usepackage{amsthm}
\usepackage{dsfont}
\usepackage{placeins}
\usepackage{booktabs}
\usepackage{caption}
\usepackage{subcaption}
\usepackage{placeins}
\newtheorem{theorem}{Theorem}    
\newtheorem{corollary}[theorem]{Corollary}      
\newtheorem{lemma}{Lemma}          

\newtheorem{proposition}[theorem]{Proposition}   
\newtheorem{assumption}{Assumption}

\usepackage[table]{xcolor}
\usepackage{pgfplots}
\pgfplotsset{compat=1.18}
\usetikzlibrary{arrows.meta,decorations.pathreplacing,positioning}

\usepackage{float}
\usepackage{booktabs}
\usepackage{multirow}
\usepackage{caption}
\usepackage{threeparttable} 
\usepackage{siunitx} 

\title{Learning-Augmented Online Allocation under Unreliable Advice: Robustness, Exposure Fairness, and Distribution Shift}

\usepackage{authblk}

\author[1\thanks{\tt{fredy-vale-manuel.pokou@inria.fr}}]{Frédy. Pokou \orcidA{}}

\affil[1]{Inria, University of Lille, CNRS, Centrale Lille Villeneuve-d’Ascq, France}

\begin{document}

\twocolumn[ 
  \begin{@twocolumnfalse} 

\maketitle

\begin{abstract}
Learning-augmented algorithms improve online decisions using predictions, but unreliable advice may harm efficiency and fairness. We study an online allocation problem with finite candidate sets, irreversible decisions, and exposure constraints. We propose a robust and fair rule combining advice with a conservative fallback and fairness correction. Under bounded-error assumptions, we prove consistency and robustness with loss proportional to prediction error. Experiments show stability under adversarial advice and significant reductions in exposure disparity.
\end{abstract}

\vspace{0.35cm}

  \end{@twocolumnfalse} 
] 



\section{Introduction}

Online allocation problems arise in recommendation, advertising, labor-market platforms, and matching systems. Classical online algorithms provide worst-case guarantees, but may be conservative; purely data-driven rules can perform well on average, but may fail under misspecification, distribution shift, or biased predictions. Learning-augmented algorithms address this tension by using predictions while retaining robustness guarantees \citep{lykouris2021competitive,purohit2018improving,mitzenmacheralgorithms}. Classical online matching and advertising allocation provide the algorithmic background \citep{karp1990optimal,mehta2007adwords}, while exposure-based fairness constraints are central in ranking and recommendation \citep{singh2018fairness}.

This paper studies whether learned advice can be used in online allocation while preserving robustness and controlling exposure imbalance. We propose a robust and fair learning-augmented rule that combines predictive advice with a conservative fallback and a virtual-queue fairness correction.

\paragraph{Contributions.}
First, we formulate a finite-horizon online allocation model with learned advice, conservative scores, and exposure targets. Second, we introduce a robust/fair learning-augmented policy. Third, we prove a central finite-sample guarantee: the robust rule is simultaneously consistent when advice is accurate and protected by a conservative fallback when advice is inaccurate. We also derive an advice-relative robustness certificate of the form
\[
\CR_T(\RLA)\ge \CR_T(\Adv)-O(\varepsilon_T),
\]
and a finite-time exposure bound for the fair rule. Fourth, we provide reproducible experiments on MovieLens-derived online allocation instances under benign noise, adversarial advice, and distribution shift.

\section{Model}

Let \(T\in\mathbb N\) be the horizon. At each time \(t\in[T]\), a request arrives and the decision-maker observes a finite feasible slate \(\calA_t\subseteq\calI\). The action \(a_t\in\calA_t\) is chosen irrevocably and yields reward \(r_t(a_t)\in[0,1]\). Each item \(i\in\calI\) belongs to a group \(g(i)\in\calG\), where \(\calG\) is finite. Before choosing, the decision-maker observes two score vectors on \(\calA_t\): a learned advice vector \(\hat r_t\) and a conservative fallback vector \(b_t\).

For a policy \(\pi\), let \(a_t^\pi\) be its action and
\begin{equation}
W_T(\pi)=\sum_{t=1}^T r_t(a_t^\pi).
\end{equation}
The offline benchmark is
\begin{equation}
\OPT_T=\sum_{t=1}^T\max_{i\in\calA_t}r_t(i),
\quad
\CR_T(\pi)=\frac{W_T(\pi)}{\OPT_T},
\end{equation}
whenever \(\OPT_T>0\). This benchmark is slate-wise and intentionally strong. In capacitated variants, \(\calA_t\) can be interpreted as the remaining feasible actions after past decisions.

The average advice and fallback errors are
\begin{equation}
\varepsilon_T=\frac1T\sum_{t=1}^T
\|\hat r_t-r_t\|_{\infty,\calA_t},
\quad
\kappa_T=\frac1T\sum_{t=1}^T
\|b_t-r_t\|_{\infty,\calA_t}.
\end{equation}
For a target exposure vector \(\rho\in\Delta(\calG)\), group exposure is
\begin{equation}
E_T^\pi(g)=\frac1T\sum_{t=1}^T\mathbf 1\{g(a_t^\pi)=g\},
\end{equation}
and the exposure gap is
\begin{equation}
\Gap_T^\rho(\pi)=\max_{g\in\calG}|E_T^\pi(g)-\rho_g|.
\end{equation}
The advice-only policy is
\begin{equation}
a_t^{\Adv}\in\arg\max_{i\in\calA_t}\hat r_t(i).
\end{equation}

\section{Algorithm}

For \(\alpha\in[0,1]\), define the robust learning-augmented score
\begin{equation}
s_t^\alpha(i)=(1-\alpha)\hat r_t(i)+\alpha b_t(i).
\end{equation}
The robust policy \(\RLA\) selects
\begin{equation}
a_t^{\RLA}\in\arg\max_{i\in\calA_t}s_t^\alpha(i).
\end{equation}

To control exposure, define virtual imbalances
\begin{equation}
Q_t(g)=\sum_{s=1}^{t-1}
\big(\mathbf 1\{g(a_s)=g\}-\rho_g\big),
\quad Q_1(g)=0.
\end{equation}
For \(\lambda\ge0\), the fair robust policy \(\FLA\) selects
\begin{equation}
a_t^{\FLA}\in
\arg\max_{i\in\calA_t}
\{s_t^\alpha(i)-\lambda Q_t(g(i))\}.
\end{equation}

\begin{algorithm}[H]
\caption{Fair Robust Learning-Augmented Allocation}
\begin{algorithmic}[1]
\Require \(\alpha\in[0,1]\), \(\lambda\ge0\), target exposure \(\rho\in\Delta(\calG)\)
\State Initialize \(Q_1(g)=0\) for all \(g\in\calG\).
\For{\(t=1,\ldots,T\)}
\State Observe \(\calA_t\), \(\hat r_t\), and \(b_t\).
\State Compute \(s_t^\alpha(i)=(1-\alpha)\hat r_t(i)+\alpha b_t(i)\).
\State Choose \(a_t\in\arg\max_{i\in\calA_t}\{s_t^\alpha(i)-\lambda Q_t(g(i))\}\).
\State Update \(Q_{t+1}(g)=Q_t(g)+\mathbf 1\{g(a_t)=g\}-\rho_g\).
\EndFor
\end{algorithmic}
\end{algorithm}

\section{Theory}

All results are deterministic conditional on the realized sequence
$(\calA_t,r_t,\hat r_t,b_t)_{t=1}^T$.

\begin{assumption}[Bounded finite slates]
\label{ass:bounded}
For all $t$ and $i\in\calA_t$,
$r_t(i),\hat r_t(i),b_t(i)\in[0,1]$, and
$1\le |\calA_t|<\infty$.
\end{assumption}

\begin{assumption}[Non-degenerate benchmark]
\label{ass:opt}
There exists $\omega>0$ such that $\OPT_T\ge \omega T$.
\end{assumption}

\begin{lemma}[Perturbation stability]
\label{lem:stability}
Let $s_t$ be any score vector on $\calA_t$, and let
$a_t^s\in\arg\max_{i\in\calA_t}s_t(i)$. Then
\[
\max_{i\in\calA_t}r_t(i)-r_t(a_t^s)
\le
2\|s_t-r_t\|_{\infty,\calA_t}.
\]
\end{lemma}

\begin{proof}
Let $i_t^\star\in\arg\max_{i\in\calA_t}r_t(i)$. Since
$s_t(a_t^s)\ge s_t(i_t^\star)$,
\[
r_t(i_t^\star)-r_t(a_t^s)
\le
|r_t(i_t^\star)-s_t(i_t^\star)|
+
|s_t(a_t^s)-r_t(a_t^s)|.
\]
The result follows by taking the maximum norm over $\calA_t$.
\end{proof}

\begin{theorem}[Central consistency-robustness bound]
\label{thm:central}
Under Assumption~\ref{ass:bounded}, for every $\alpha\in[0,1]$,
\[
\OPT_T-W_T(\RLA)
\le
2T\big((1-\alpha)\varepsilon_T+\alpha\kappa_T\big).
\]
Consequently, under Assumption~\ref{ass:opt},
\[
\CR_T(\RLA)
\ge
1-\frac{2}{\omega}
\big((1-\alpha)\varepsilon_T+\alpha\kappa_T\big).
\]
\end{theorem}

\begin{proof}
For each $t$,
\[
\|s_t^\alpha-r_t\|_{\infty,\calA_t}
\le
(1-\alpha)\|\hat r_t-r_t\|_{\infty,\calA_t}
+
\alpha\|b_t-r_t\|_{\infty,\calA_t}.
\]
Apply Lemma~\ref{lem:stability} with $s_t=s_t^\alpha$, sum over $t$, and divide by $\OPT_T\ge\omega T$.
\end{proof}

\begin{corollary}[Consistency]
\label{cor:consistency}
If $\alpha=0$, then
\[
\CR_T(\Adv)\ge 1-\frac{2\varepsilon_T}{\omega}.
\]
In particular, perfect advice $(\varepsilon_T=0)$ is offline-optimal.
\end{corollary}

\begin{corollary}[Robust fallback protection]
\label{cor:fallback}
If $\alpha=1$, then
\[
\CR_T(\RLA)\ge 1-\frac{2\kappa_T}{\omega}.
\]
Thus the policy remains protected whenever the conservative score has bounded error, irrespective of the advice error.
\end{corollary}

\begin{corollary}[Advice-relative certificate]
\label{cor:certificate}
Under Assumptions~\ref{ass:bounded}--\ref{ass:opt},
\[
\CR_T(\RLA)
\ge
\CR_T(\Adv)
-
\frac{2\alpha}{\omega}(\kappa_T+\varepsilon_T).
\]
Hence, if $\kappa_T=O(\varepsilon_T)$, then
\[
\CR_T(\RLA)\ge \CR_T(\Adv)-O(\varepsilon_T).
\]
\end{corollary}

\begin{proof}
By Corollary~\ref{cor:consistency},
$\CR_T(\Adv)\le 1$ and its loss from one is at most
$2\varepsilon_T/\omega$. Theorem~\ref{thm:central} gives the
corresponding lower bound for $\RLA$. Combining the two inequalities yields
the claim.
\end{proof}

\begin{assumption}[Corrective availability]
\label{ass:correct}
There exist $\delta>0$ and $\Delta\in[0,1]$ such that, whenever
$Q_t(g)-Q_t(h)>\delta T$, any available item from the over-exposed group
$g$ can be replaced by an available item from group $h$ whose robust score
is lower by at most $\Delta$.
\end{assumption}

\begin{theorem}[Finite-time exposure control]
\label{thm:fairness}
Under Assumptions~\ref{ass:bounded} and~\ref{ass:correct}, if
\[
\lambda>\frac{\Delta}{\delta T},
\]
then
\[
\max_{g,h\in\calG}\{Q_{T+1}(g)-Q_{T+1}(h)\}
\le \delta T+2.
\]
Consequently,
\[
\Gap_T^\rho(\FLA)\le \delta+\frac{2}{T}.
\]
\end{theorem}

\begin{proof}
Suppose $Q_t(g)-Q_t(h)>\delta T$. By Assumption~\ref{ass:correct}, an available group-$g$ item can be replaced by a group-$h$ item with robust-score loss at most $\Delta$. The penalized-score advantage of the group-$g$ item is then at most
\[
\Delta-\lambda(Q_t(g)-Q_t(h))<0.
\]
Thus the policy cannot increase an already excessive pairwise imbalance. Since one decision changes any pairwise imbalance by at most two, the queue bound follows. Finally,
$Q_{T+1}(g)=T(E_T(g)-\rho_g)$, which gives the exposure-gap bound.
\end{proof}

\begin{proposition}[Incentive dampening]
\label{prop:strategy}
If an item can change its advice score by at most $m\ge0$, then its one-period score gain under $\RLA$ or $\FLA$ is at most $(1-\alpha)m$. Hence any manipulation of size $m$ with cost larger than $(1-\alpha)m$ is unprofitable.
\end{proposition}

\begin{proof}
The advice enters the decision score only through the coefficient $1-\alpha$. A perturbation of magnitude $m$ can therefore change the score by at most $(1-\alpha)m$.
\end{proof}

\section{Computational study}
\label{sec5}
We evaluate the proposed policies on online allocation instances derived from the MovieLens 1M data set. The raw data contain user-movie ratings. We interpret each arriving user as an online request and the available movies as the feasible slate. Ratings are rescaled to $[0,1]$ and used as realized rewards. At each period, the policy selects one movie from the slate irrevocably.

\paragraph{Instance construction.}
A matrix-factorization model is trained on a fixed training split and used to generate the advice vector $\hat r_t$. The conservative score $b_t$ is a popularity-calibrated score, adjusted to avoid over-reliance on the learned predictor. Candidate slates contain both high-score items and lower-quality decoy items, so that random and popularity policies are nontrivial but not artificially favored. Movie groups define the exposure categories, and the target vector $\rho$ is set to the empirical group distribution in the candidate pool. All results are averaged over independent arrival sequences generated with fixed random seeds.

\paragraph{Stress regimes.}
We consider three regimes. In the benign regime, advice is perturbed by mean-zero noise of level $\sigma$. In the adversarial-advice regime, advice is systematically biased across exposure groups, mimicking strategic or discriminatory score distortion. In the distribution-shift regime, test arrivals over-sample a subpopulation whose preferences differ from the training distribution. These regimes are designed to test the consistency--robustness trade-off predicted by Theorem~\ref{thm:central} and Corollary~\ref{cor:certificate}.

\paragraph{Policies and metrics.}
We compare RANDOM, POPULARITY, ADVICE, ROBUST-LA, and FAIR-LA. Performance is measured by the empirical competitive ratio $\CR_T$ relative to the slate-wise offline benchmark and by the exposure fairness gap $\Gap_T^\rho$. For the robustness certificate, we also report the empirical advice error $\varepsilon_T$ and the gain in competitive ratio relative to ADVICE.

\begin{figure}[H]
    \centering
    \includegraphics[width=\linewidth]{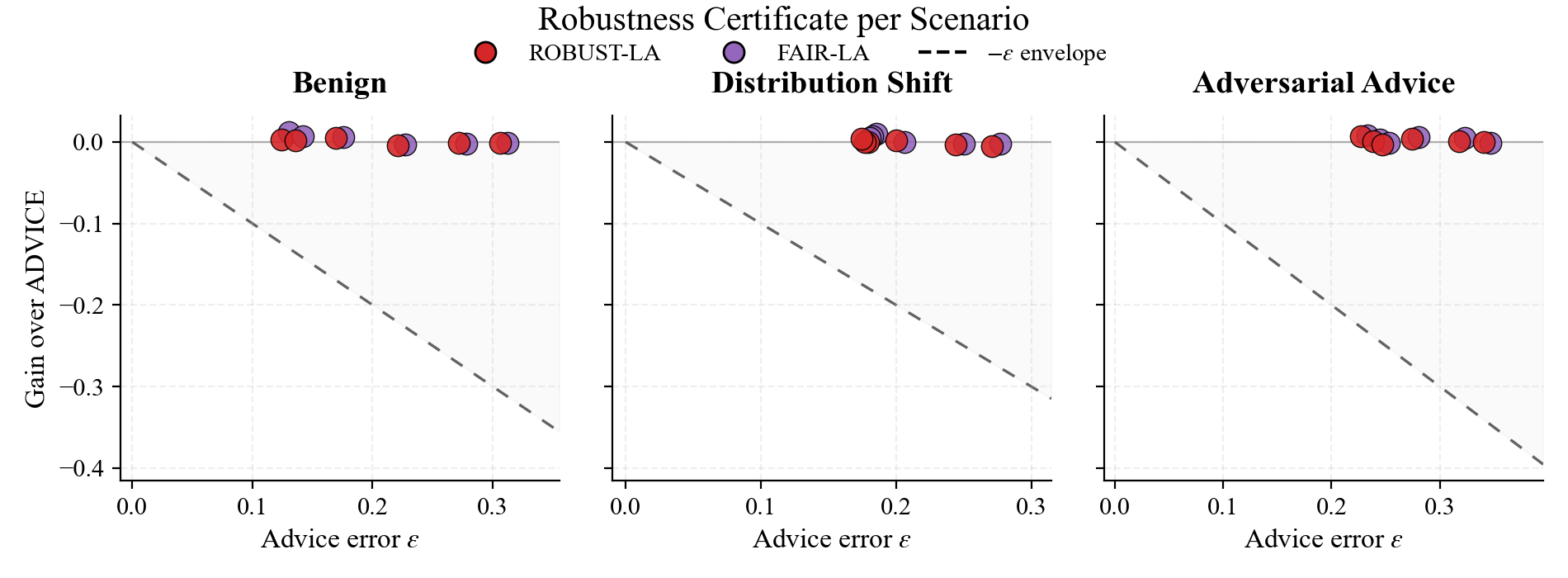}
    \caption{Robustness certificate across scenarios. The vertical axis reports the competitive-ratio gain over ADVICE and the horizontal axis reports the empirical advice error $\varepsilon_T$. The dashed line is the $-\varepsilon_T$ envelope. The points remain above this envelope, in line with the error-dependent guarantee in Corollary~\ref{cor:certificate}.}
    \label{fig:robustness_certificate}
\end{figure}

Figure~\ref{fig:robustness_certificate} provides the empirical counterpart of the advice-relative robustness certificate. Across benign noise, distribution shift, and adversarial advice, ROBUST-LA and FAIR-LA remain close to or above the advice-only rule, with no collapse as advice error increases. This supports the interpretation that conservative interpolation prevents catastrophic degradation when predictions are unreliable.

\begin{figure}[H]
    \centering
    \includegraphics[width=\linewidth]{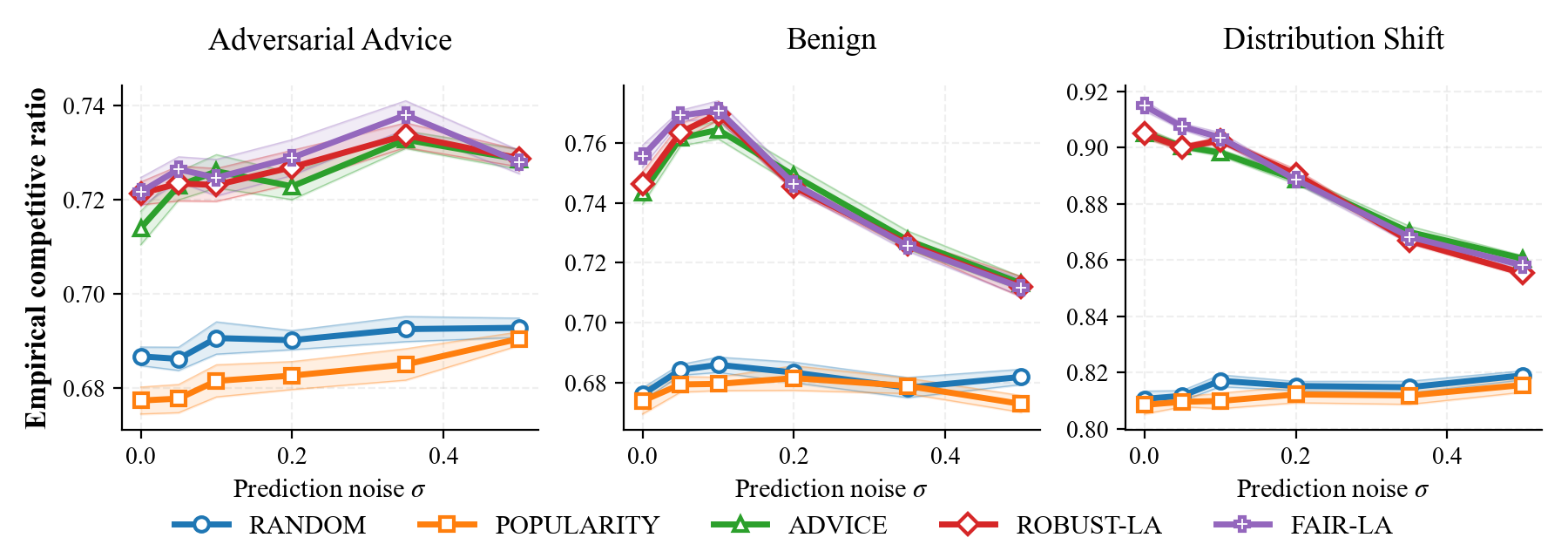}
    \caption{Competitive ratio under prediction noise. ROBUST-LA and FAIR-LA preserve high competitive ratios across the three stress regimes, while RANDOM and POPULARITY are consistently separated from the learning-augmented policies.}
    \label{fig:competitive_ratio_noise}
\end{figure}

Figure~\ref{fig:competitive_ratio_noise} shows that the learning-augmented policies retain a clear efficiency advantage over RANDOM and POPULARITY. In the adversarial and benign regimes, the proposed rules remain stable as $\sigma$ varies. Under distribution shift, competitive ratios remain high for all learning-augmented policies, while non-personalized baselines stay near $0.81$.

\begin{figure}[H]
    \centering
    \includegraphics[width=\linewidth]{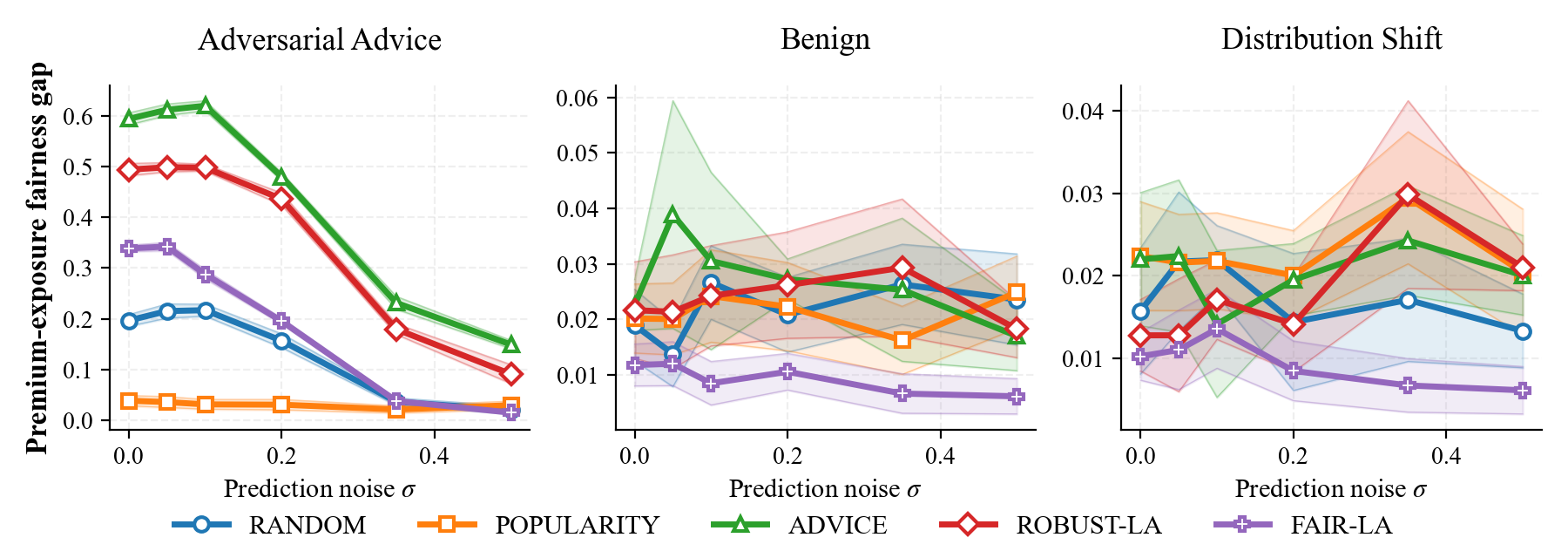}
    \caption{Exposure fairness gap under prediction noise. FAIR-LA sharply reduces exposure imbalance, especially under adversarial advice, where the advice-only policy induces large disparities.}
    \label{fig:fairness_gap_noise}
\end{figure}

Figure~\ref{fig:fairness_gap_noise} confirms the role of the virtual exposure correction. Under adversarial advice, ADVICE produces severe exposure imbalance, whereas FAIR-LA reduces the gap substantially. In benign and shifted regimes, FAIR-LA also delivers the lowest or near-lowest exposure gap over most noise levels. POPULARITY can occasionally have a small fairness gap, but this is obtained with much lower competitive ratio; hence it does not provide the same efficiency-fairness trade-off.

\begin{figure}[H]
    \centering
    \includegraphics[width=\linewidth]{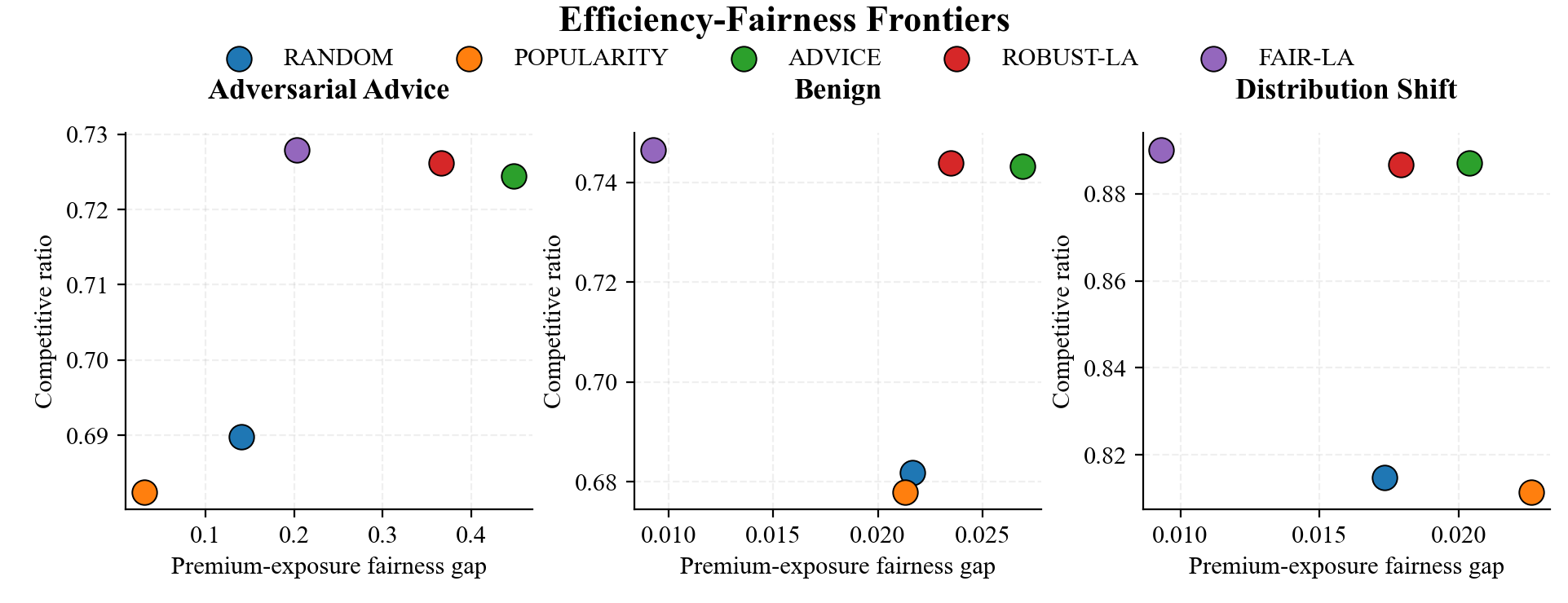}
    \caption{Efficiency--fairness frontiers. Each point reports the average competitive ratio and exposure fairness gap of one policy. FAIR-LA lies on the favorable frontier by combining high efficiency with low exposure disparity.}
    \label{fig:efficiency_fairness_frontier}
\end{figure}

Figure~\ref{fig:efficiency_fairness_frontier} summarizes the trade-off. FAIR-LA is the most stable policy on the efficiency--fairness frontier: it sacrifices little competitive ratio relative to ROBUST-LA or ADVICE while achieving substantially lower exposure disparity. This is the main empirical message of the study.

\begin{table}[h]
\centering
\resizebox{8.3cm}{!}{
\begin{tabular}{llcc@{\hspace{2em}}cc@{\hspace{2em}}cc}
\toprule
& & \multicolumn{2}{c}{\textbf{Benign}} & \multicolumn{2}{c}{\textbf{Dist. Shift}} & \multicolumn{2}{c}{\textbf{Adversarial}} \\
\cmidrule(lr){3-4} \cmidrule(lr){5-6} \cmidrule(lr){7-8}
\textbf{Noise ($\sigma$)} & \textbf{Policy} & \textbf{CR} $\uparrow$ & \textbf{FG} $\downarrow$ & \textbf{CR} $\uparrow$ & \textbf{FG} $\downarrow$ & \textbf{CR} $\uparrow$ & \textbf{FG} $\downarrow$ \\
\midrule
\multirow{5}{*}{0.00} 
& ADVICE      & 0.743 & 0.022          & 0.905 & 0.022          & 0.714 & 0.594 \\
& POPULARITY  & 0.674 & 0.020          & 0.808 & 0.022          & 0.677 & \textbf{0.038} \\
& RANDOM      & 0.676 & 0.019          & 0.811 & 0.016          & 0.687 & 0.196 \\
& ROBUST-LA   & 0.746 & 0.022          & 0.905 & 0.013          & 0.721 & 0.494 \\
& FAIR-LA     & \textbf{0.755} & \textbf{0.012} & \textbf{0.915} & \textbf{0.010} & \textbf{0.722} & 0.339 \\
\midrule
\multirow{5}{*}{0.10} 
& ADVICE      & 0.764 & 0.030          & 0.898 & 0.014          & \textbf{0.726} & 0.620 \\
& POPULARITY  & 0.680 & 0.024          & 0.810 & 0.022          & 0.681 & \textbf{0.031} \\
& RANDOM      & 0.686 & 0.027          & 0.817 & 0.022          & 0.690 & 0.217 \\
& ROBUST-LA   & 0.769 & 0.024          & 0.903 & 0.017          & 0.723 & 0.498 \\
& FAIR-LA     & \textbf{0.771} & \textbf{0.008} & \textbf{0.904} & \textbf{0.013} & 0.725 & 0.288 \\
\midrule
\multirow{5}{*}{0.50} 
& ADVICE      & 0.713 & 0.017          & 0.860 & 0.020          & 0.729 & 0.149 \\
& POPULARITY  & 0.673 & 0.025          & 0.816 & 0.020          & 0.690 & 0.029 \\
& RANDOM      & 0.682 & 0.024          & \textbf{0.819} & 0.013          & 0.693 & 0.020 \\
& ROBUST-LA   & 0.712 & 0.018          & 0.855 & 0.021          & \textbf{0.729} & 0.091 \\
& FAIR-LA     & 0.712 & \textbf{0.006} & 0.858 & \textbf{0.006} & 0.728 & \textbf{0.015} \\
\bottomrule
\end{tabular}
}
\caption{\textbf{Performance Comparison across Scenarios.} We report the mean Competitive Ratio (CR $\uparrow$) and Fairness Gap (FG $\downarrow$). Robust-LA and Fair-LA consistently bridge the gap between pure Advice and baseline policies.}
\label{tab:results_summary}
\end{table}

Table~\ref{tab:results_summary} reports representative noise levels. The table highlights two points. First, RANDOM and POPULARITY are not competitive in efficiency: they are consistently below the learning-augmented policies in CR. Second, FAIR-LA achieves the most reliable fairness improvement. In adversarial settings, ADVICE can achieve high CR but at the cost of extreme exposure gaps; FAIR-LA substantially reduces this disparity while preserving nearly the same CR.

\section{Conclusion}

We studied online allocation with learned advice, conservative fallback scores, and exposure-fairness targets. The proposed rule is deliberately simple: interpolate between advice and fallback scores, then penalize cumulative exposure imbalance. This simplicity yields finite-horizon guarantees. The central bound shows that performance degrades with a weighted combination of advice error and fallback error; the corollaries recover consistency, fallback protection, and the advice-relative certificate $\CR_T(\RLA)\ge \CR_T(\Adv)-O(\varepsilon_T)$. A virtual-queue argument gives finite-time exposure control.

The computational study supports these conclusions on MovieLens-derived online allocation instances. ROBUST-LA protects efficiency under unreliable advice, while FAIR-LA provides the strongest efficiency--fairness compromise. In particular, FAIR-LA sharply reduces exposure disparity under adversarial advice without collapsing in competitive ratio. Future work may extend the analysis to hard matching capacities and to endogenous, strategically generated advice.

\section*{Data Availability}

All numerical experiments in this study are based on synthetic benchmark environments generated algorithmically by the authors. 

\section*{Code Availability}

The Python code used to generate the benchmark environments, compute the optimal policies via dynamic programming, train all boundary-based and reinforcement-learning baselines, and reproduce the tables and figures is available from the corresponding author upon reasonable request.

\FloatBarrier
\setcounter{figure}{0}
\renewcommand{\thefigure}{\Alph{section}\arabic{figure}}

\setcounter{table}{0}
\renewcommand{\thetable}{\Alph{section}\arabic{table}}



\normalsize
\bibliography{references}


\end{document}